\documentclass{article}

\usepackage{microtype}
\usepackage{subcaption}
\usepackage{booktabs} 
\usepackage{multirow}
\usepackage[table,xcdraw]{xcolor}
\usepackage{graphicx}
\definecolor{lightgray}{gray}{0.9}
\usepackage{fontawesome5} 
\newcommand{\CLEAR}{\textbf{\texttt{\textcolor{blue!75!black}{CLEAR}}}}
\newcommand{\Strict}{\textbf{\texttt{\textcolor{gray!70!black}{Strict}}}}
\newcommand{\Surge}{\textbf{\texttt{\textcolor{red!75!black}{Surge}}}}
\newcommand{\Ample}{\textbf{\texttt{\textcolor{orange!90!black}{Ample}}}}

\DeclareRobustCommand{\Uniform}{\textbf{\texttt{\textcolor{gray!65!black}{Uniform}}}}
\DeclareRobustCommand{\Predictor}{\textbf{\texttt{\textcolor{gray!40!black}{Predictor}}}}
\DeclareRobustCommand{\Oracle}{\textbf{\texttt{\textcolor{red!70!black}{Oracle}}}}
\DeclareRobustCommand{\TALEEP}{\textbf{\texttt{\textcolor{orange!85!black}{TALE-EP}}}}

\usepackage{hyperref}

\usepackage[accepted]{icml2026}

\usepackage{amsmath}
\usepackage{amssymb}
\usepackage{mathtools}
\usepackage{amsthm}

\usepackage[capitalize,noabbrev]{cleveref}

\theoremstyle{plain}
\newtheorem{theorem}{Theorem}[section]

\theoremstyle{definition}
\newtheorem{definition}[theorem]{Definition}

\theoremstyle{remark}

\usepackage[textsize=tiny]{todonotes}

\icmltitlerunning{Economic Perspective on Optimal Budget Allocation for LLMs}

\begin{document}

\twocolumn[
  \icmltitle{The Shadow Price of Reasoning:\\
  Economic Perspective on Optimal Budget Allocation for LLMs
}


  \icmlsetsymbol{equal}{*}

  \begin{icmlauthorlist}
    \icmlauthor{Xu Wan}{equal,1,2,3}
    \icmlauthor{Speed Zhu}{equal,2}
    \icmlauthor{Jianwei Cai}{2}
    \icmlauthor{Guang Chen}{2}
    \icmlauthor{Ximing Huang}{3}
    \icmlauthor{Wiggin Zhou}{2}
    \icmlauthor{Mingyang Sun}{3}
  \end{icmlauthorlist}
  \icmlaffiliation{1}{Zhejiang University}
  \icmlaffiliation{2}{Tecent HY Team}
  \icmlaffiliation{3}{Peking University}

  \icmlcorrespondingauthor{Mingyang Sun}{smy@pku.edu.cn}

  \icmlkeywords{Machine Learning, ICML}

  \vskip 0.3in
]



\printAffiliationsAndNotice{}  

\begin{abstract}
        Inference-time scaling has emerged as a critical avenue for enhancing Large Language Models' performance, yet real-world deployment is constrained by strict computational budgets. In this work, we formulate inference budget allocation as a global constrained optimization problem governed by economic principles. By modeling per-query reasoning utility with a shifted-surge function, we derive an optimal allocation policy based on a global shadow price that equilibrates marginal utility under resource scarcity.
        Based on this theory, we propose \textbf{\texttt{C}}onstrained \textbf{\texttt{L}}atent-utility \textbf{\texttt{E}}quilibrium \textbf{\texttt{A}}llocation for \textbf{\texttt{R}}easoning (\CLEAR{}). It performs rational abandonment and reallocates resources from insolvent queries to solvable queries near their emergence thresholds.
        Extensive experiments on several reasoning tasks with different traffic streams demonstrate that \CLEAR{} significantly improves the Pareto frontier of total token cost versus mean accuracy. In resource-scarce regimes, \CLEAR{} achieves up to a \textbf{3$\times$} improvement in global accuracy compared to uniform allocation. 

    \faGithub \ The code is available in \href{https://github.com/waunx/CLEAR}{\textbf{\texttt{Here}}}.
\end{abstract}

\section{Introduction}
\label{sec:intro}

\begin{figure}[t]
    \centering
    \includegraphics[width=0.95\linewidth]{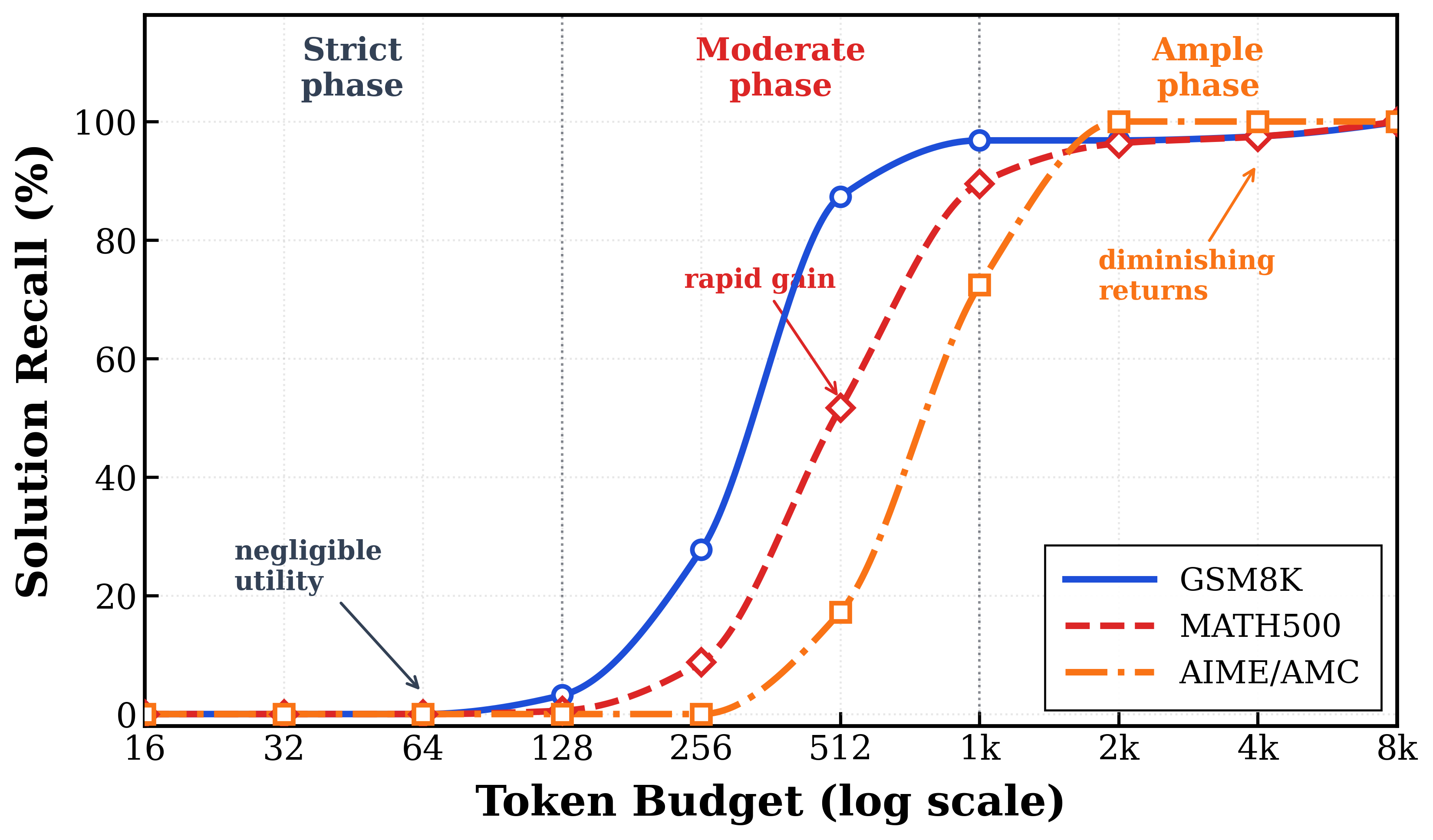}
    \caption{\textbf{The S-Shaped Compute-Utility Curve.} We evaluate \texttt{Qwen2.5-Math-7B}~\citep{yang2024qwen2} on three benchmarks. The budget-performance relationship exhibits three distinct regions: (1) a pre-threshold \Strict{} phase with negligible utility; (2) a rapid \Surge{} phase offering high leverage; and (3) an \Ample{} phase characterized by diminishing returns.}
    \label{fig:s_curve_regions}
\end{figure}

\begin{figure*}[htbp]
    \centering
    \includegraphics[width=\linewidth]{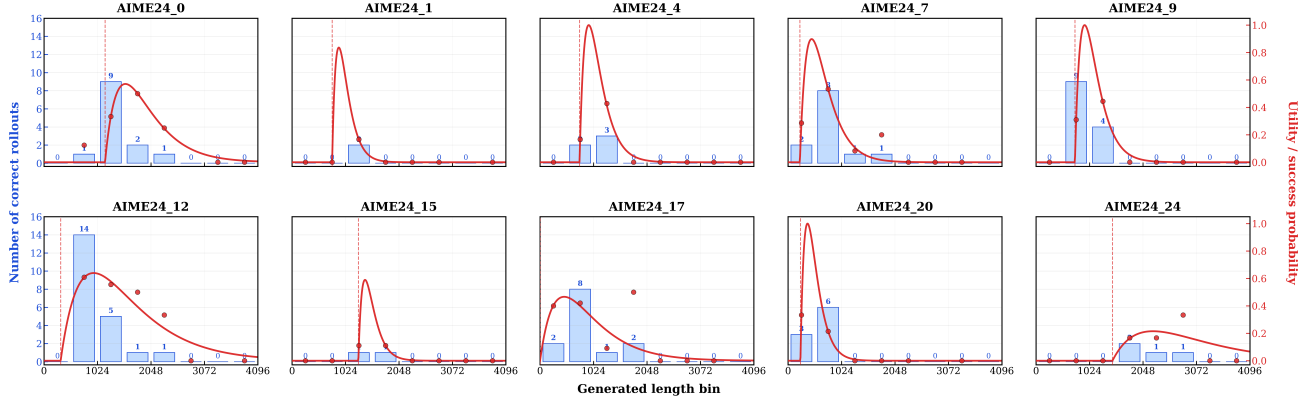} 
    \caption{\textbf{Empirical Rollouts and Latent Utility.} For selected \texttt{AIME-24} problems, blue bars show the number of correct rollouts in each length bin, while red curves depict the fitted latent utility mapping induced by our shifted-surge model.}
    \label{fig:motivation_scaling}
\end{figure*}

The development of Large Language Models (LLMs) is undergoing a paradigm shift from training-time scaling to inference-time scaling~\citep{wei2022chain, snell2024scaling, wu2024inference}. Recent breakthroughs demonstrate that increasing test-time compute can yield performance gains comparable to scaling model parameters by invoking ``System 2'' reasoning capabilities~\citep{weston2023system, brown2024large}. A series of length-scaling experiments has furthermore validated that by allowing models to think longer, it leads to substantial improvements in complex reasoning tasks~\citep{muennighoff2025s1, aggarwal2025l1}. However, in real-world deployment scenarios, test-time compute is a finite and expensive commodity. Whether serving millions of concurrent users via cloud APIs or running models on resource-constrained edge devices, practitioners operate under strict global budget constraints~\citep{chen2023frugalgpt}. The central challenge thus shifts from merely increasing the theoretical ceiling of model intelligence to maximizing global utility under a fixed computational cap~\citep{raposo2024mixture}.

Standard deployment practices usually impose a {uniform policy}, assigning the same generation limit (e.g., a fixed \texttt{max\_new\_tokens}) to every query. This implicitly assumes that all queries share a comparable compute-utility profile, which conflicts with the large heterogeneity of reasoning tasks. As illustrated in Figure~\ref{fig:s_curve_regions}, reasoning utility follows an S-shaped curve when we fix the decoding strategy and vary the token limit. For difficult queries like \texttt{AIME}, a uniform limit may leave the trajectory in the \Strict{} phase, where compute is spent but utility remains near zero. For easy queries like \texttt{GSM8K}, the same limit may push generation into the \Ample{} phase, where additional tokens bring only diminishing returns. Of course, in an unconstrained-resource setting, the inefficiency caused by pushing easy queries into the \Ample{} phase may be less consequential. However, this work focuses on the resource-constrained setting, where a fixed total token budget is imposed over the entire evaluation set. The central question is therefore how to allocate tokens across individual queries so as to maximize overall token utility under this global budget.

We therefore treat inference budgeting as a batch-level allocation problem: given a fixed total token supply, decide how many tokens each query should receive to maximize aggregate expected utility. This gives rise to a global constrained optimization problem over instance-specific utility curves.
Although the resulting objective is non-convex, its Lagrangian form reveals a simple economic principle: at optimum, every active query should spend tokens until its marginal utility matches a common \textit{global shadow price}~\citep{boyd2004convex, devanur2019near}. Queries whose attainable marginal gain never clears this price should receive no budget, while the remaining queries share the budget according to their utility slopes.

Building on these findings, we propose \textbf{\texttt{C}}onstrained \textbf{\texttt{L}}atent-utility \textbf{\texttt{E}}quilibrium \textbf{\texttt{A}}llocation for \textbf{\texttt{R}}easoning (\CLEAR{}). Instead of a uniform utility for each query, \CLEAR{} maps it into a surge-shaped utility curve. The system then operates as a computational market:
(1) \textbf{Threshold Modeling:} We estimate the emergence threshold of each query and instantiate its latent utility curve;
(2) \textbf{Price Discovery:} We employ a fast bisection search to find the unique global shadow price that clears the market, ensuring the total demand matches the available budget;
(3) \textbf{Optimal Allocation:} Based on the discovered price, we apply a closed-form policy derived from the Lambert W function to strictly determine the token limit for each query, automatically handling truncation or rational abandonment.
Crucially, \CLEAR{} requires no retraining of the backbone LLM and operates as a plug-and-play inference wrapper.

Our contributions are summarized as follows:

(1) We identify a three-phase compute-utility pattern in LLM reasoning and formalize inference-time token allocation over non-concave latent utility curves, where the optimal policy is governed by a global shadow price.
  
(2) We derive a closed-form Lambert W allocation policy and instantiate it in \CLEAR{}, a plug-and-play framework that combines latent-threshold prediction with market-clearing price discovery.
    
(3) We validate \CLEAR{} on mixed-complexity mathematical reasoning benchmarks, demonstrating improved cost-accuracy Pareto efficiency and robustness to hyperparameter choices and predictor noise.

\section{Empirical Motivation}
\label{sec:motivation}
To better model the per-query utility of reasoning tokens, we investigate whether longer reasoning is uniformly beneficial for each query. If every problem has an instance-specific favorable reasoning length, then inference utility should naturally rise after a minimum threshold but eventually saturate or decline when generation becomes excessive.

Following the thoughtology analysis~\citep{marjanovic2025deepseek} in \texttt{DeepSeek-R1}, we conduct a controlled experiment using \texttt{Qwen2.5-Math-7B}. We perform sampling with a high temperature $T=1.0$ to induce diverse reasoning paths of varying lengths. We generate $N=50$ responses for \texttt{AIME-24} and $N=4$ for the \texttt{GSM8K}, \texttt{MATH-500} benchmarks, and group the resulting trajectories into length bins to compute the conditional \texttt{Pass@1} accuracy.

As visualized in Figure~\ref{fig:motivation_scaling}, reasoning utility exhibits non-linear dynamics rather than simple scaling. We identify three recurring regimes: a \Strict{} phase where short trajectories fail consistently, implying a minimum solvable threshold; a \Surge{} phase where performance rises sharply past this threshold; and an \Ample{} phase where additional generation yields diminishing returns and may eventually degrade solution quality. 

\section{Problem Formulation}
\label{sec:formulation}
Based on the empirical observations above, we first model the per-query reasoning utility with a function that captures the \Strict{}--\Surge{}--\Ample{} structure. Under this utility model, the problem of assigning tokens subject to a fixed total budget naturally becomes a global constrained optimization problem with an instance-dependent, non-concave utility landscape.

\subsection{Modeling the Physics of Reasoning}
 While the observed outcome of a reasoning task is binary (correct or incorrect), we posit that this outcome is governed by a continuous, unobservable variable: the {reasoning utility} $\phi(t)$. For a query $s_i$, we model the accumulated potential after generating $t$ tokens follows a Shifted Surge Function:

\begin{equation}
    \phi_i(t) = \begin{cases} 
    0 & 0 \le t < \tau_i \\
    \alpha_i (t - \tau_i) \cdot e^{-\beta_i (t - \tau_i)} & t \ge \tau_i 
    \end{cases}
    \label{eq:surge_potential}
\end{equation}
This parameterization aligns the latent utility model with the three empirical regimes observed above:

\textbf{\Strict{}:} Utility remains zero until the generation length crosses the latent emergence threshold $\tau_i$.

\textbf{\Surge{}:} Once the threshold is crossed, reasoning potential rises with initial velocity $\alpha_i$, capturing the rapid accumulation of valid reasoning.

\textbf{\Ample{}:} When generation extends excessively, the exponential decay term $e^{-\beta_i \Delta t}$ dominates, reflecting diminishing returns.

In Figure~\ref{fig:motivation_scaling}, the red curve visualizes the fitted latent reasoning potential, which provides a continuous utility mapping from the empirical rollout outcomes observed at different generation lengths.

\subsection{The Global Optimization Objective}
We adopt a pure token-based cost model where the cost $C_i(t) = t$. The system's goal is to allocate a vector of tokens $\mathbf{t} = [t_1, \dots, t_N]$ for $N$ queries to maximize aggregate reasoning potential under a global token budget limit $B_{\text{total}}$:

\begin{equation}
\begin{aligned}
    \max_{\mathbf{t} \in \mathbb{R}_{\ge 0}^N} \quad & \sum_{i=1}^N \phi_i(t_i) \\
    \text{s.t.} \quad & \sum_{i=1}^N t_i \le B_{\text{total}}.
\end{aligned}
\label{eq:global_opt}
\end{equation}

\section{Theoretical Analysis}
\label{sec:theory}

To solve the non-convex optimization problem in Eq.~\eqref{eq:global_opt}, we apply the method of Lagrange multipliers. By relaxing the global budget constraint, we establish an economic principle governing optimal inference: the equalization of marginal reasoning potential.

\subsection{Shadow Price Parity}
We construct the Lagrangian $\mathcal{L}(\mathbf{t}, \lambda) = \sum_{i} \phi_i(t_i) - \lambda (\sum_{i} t_i - B_{\text{total}})$, where $\lambda \ge 0$ is the Lagrange multiplier. In economic terms, $\lambda$ represents the \textit{global shadow price} of computation—the marginal gain in total potential obtained by relaxing the budget by one token.

The Karush-Kuhn-Tucker (KKT) conditions imply that for any globally optimal allocation $\mathbf{t}^*$, the marginal gain of each active task must align with this global price. This parity determines whether a query should be lifted out of \Strict{}, placed near its high-leverage \Surge{}, or capped before wasting budget in \Ample{}. We formalize this as:

\begin{definition}[Shadow Price Parity]
\label{def:shadow_price_parity}
A strictly positive budget allocation $t_i^* > 0$ is a local maximum if it satisfies the first-order stationarity condition, equalizing the marginal reasoning potential to the global shadow price:
\begin{equation}
    \frac{\partial \phi_i(t_i^*)}{\partial t_i} = \lambda.
\end{equation}
The {global optimal policy} is then determined by comparing the utility surplus at this stationary point against the abandonment option ($t_i=0$). We term this the \textit{Rational Abandonment Condition}, where tasks whose maximum net surplus falls below zero are deemed economically insolvent and assigned zero budget.
\end{definition}

\subsection{Individual Optimal Allocation Policy}
For the surge potential $\phi_i(t) = \alpha_i \Delta t e^{-\beta_i \Delta t}$, where $\Delta t = t - \tau_i$, the shadow price parity condition yields the differential equation:
\begin{equation}
    \frac{d \phi_i}{d t} = \alpha_i e^{-\beta_i \Delta t} (1 - \beta_i \Delta t) = \lambda.
\end{equation}
While standard algebraic methods cannot solve for $t$ in this transcendental equation, we prove that the exact solution is given by the {Lambert W function} $W(z)$, defined as the inverse of $f(w) = we^w$.

\begin{theorem}[Individual Optimal Allocation Policy]
\label{thm:lambert}
Under the shadow price parity condition, the optimal token allocation $t_i^*$ for a given $\lambda$ follows a closed-form solution:
\begin{equation}
    t_i^*(\lambda) = \tau_i + \frac{1}{\beta_i} \left[ 1 - W_0\left( \frac{\lambda e}{\alpha_i} \right) \right],
    \label{eq:lambert_solution}
\end{equation}
subject to the solvency constraint $\phi_i(t_i^*) > \lambda t_i^*$. Here, $W_0(\cdot)$ denotes the principal branch of the Lambert W function, and $e$ is Euler's number.
\end{theorem}

\begin{proof}
The objective for each task $i$ is to maximize the net economic surplus $J_i(t)$, defined as the difference between the latent reasoning potential and the opportunity cost of tokens:
\begin{equation}
    J_i(t) = \phi_i(t) - \lambda t = \alpha_i (t - \tau_i) e^{-\beta_i (t - \tau_i)} - \lambda t.
\end{equation}
We assume the task is in the active reasoning phase ($t \ge \tau_i$). Let $\Delta t = t - \tau_i$ denote the effective generation length. The first-order necessary condition for optimality requires the marginal potential to equal the shadow price:
\begin{equation}
    \frac{d \phi_i}{dt} = \lambda.
\end{equation}
Applying the product rule to differentiate the surge function with respect to $t$:
\begin{equation}
    \begin{aligned}
        \frac{d}{dt} \left[ \alpha_i \Delta t e^{-\beta_i \Delta t} \right] &= \alpha_i \left( 1 \cdot e^{-\beta_i \Delta t} + \Delta t \cdot (-\beta_i) e^{-\beta_i \Delta t} \right) \\
        &= \alpha_i e^{-\beta_i \Delta t} (1 - \beta_i \Delta t).
    \end{aligned}
\end{equation}
Equating this to $\lambda$, we obtain the transcendental equation:
\begin{equation}
    \alpha_i e^{-\beta_i \Delta t} (1 - \beta_i \Delta t) = \lambda.
\end{equation}
To solve for $\Delta t$, we transform this equation into the canonical form of the Lambert W function, $w e^w = z$. We introduce a change of variable:
\begin{equation}
    u = 1 - \beta_i \Delta t.
\end{equation}
This implies $\beta_i \Delta t = 1 - u$, and consequently, the exponential term becomes $e^{-\beta_i \Delta t} = e^{-(1-u)} = e^{u-1}$. Substituting these into the optimality condition:
\begin{equation}
    \begin{aligned}
        \alpha_i \cdot e^{u-1} \cdot u &= \lambda \\
        u e^u e^{-1} &= \frac{\lambda}{\alpha_i} \\
        u e^u &= \frac{\lambda e}{\alpha_i}.
    \end{aligned}
\end{equation}
By the definition of the Lambert W function, the solution for $u$ is given by the principal branch $W_0$, as we require the solution corresponding to the physical expansion phase:
\begin{equation}
    u = W_0\left( \frac{\lambda e}{\alpha_i} \right).
\end{equation}
Finally, we substitute $u$ back to solve for the optimal allocation $t_i^*$:
\begin{equation}
    \begin{aligned}
        1 - \beta_i (t_i^* - \tau_i) &= W_0\left( \frac{\lambda e}{\alpha_i} \right) \\
        \beta_i (t_i^* - \tau_i) &= 1 - W_0\left( \frac{\lambda e}{\alpha_i} \right) \\
        t_i^* &= \tau_i + \frac{1}{\beta_i} \left[ 1 - W_0\left( \frac{\lambda e}{\alpha_i} \right) \right].
    \end{aligned}
\end{equation}
This interior solution is valid if and only if the resulting net surplus is positive ($J_i(t_i^*) > 0$). If the surplus is non-positive, the global optimum defaults to the boundary solution $t_i^* = 0$.
\end{proof}

\subsection{Global Equilibrium and Phase Transitions}
\label{sec:phase_transition}

The Lambert W policy derived in Eq.~\eqref{eq:lambert_solution} describes the local optimum for a fixed price $\lambda$. However, the inference system operates under a hard global constraint. The \textit{global optimal policy} is realized by finding the unique market-clearing price $\lambda^*$ that satisfies the resource budget:
\begin{equation}
    \sum_{i=1}^N t_i^*(\lambda^*) = B_{\text{total}}.
    \label{eq:market_clearing}
\end{equation}
Since $t_i^*(\lambda)$ is strictly monotonically decreasing with respect to $\lambda$, a unique $\lambda^*$ is guaranteed to exist. The magnitude of the global budget $B_{\text{total}}$ inversely determines $\lambda^*$. This induces three system-level allocation regimes, which correspond to different operating points on the per-query \Strict{}--\Surge{}--\Ample{} utility curves:

\textbf{1. Abundant-Budget Regime.}
When $B_{\text{total}}$ is large, the shadow price $\lambda^*$ approaches zero. In this limit, the Lambert W term vanishes: $W_0(0)=0$, and the allocation converges to $t_i \to \tau_i + 1/\beta_i$. Active tasks are therefore allowed to approach the peak of their utility curves, near the transition from \Surge{} to \Ample{}.

\textbf{2. Scarce-Budget Regime.}
As $B_{\text{total}}$ tightens, $\lambda^*$ rises. The term $W_0(\frac{\lambda e}{\alpha_i})$ increases, compressing the allocation $t_i^*$ toward the threshold $\tau_i$. The system thus prioritizes queries whose allocated tokens can still land in the high-marginal-return \Surge{} segment.

\textbf{3. Abandonment Regime.}
When $B_{\text{total}}$ is critically low, $\lambda^*$ exceeds the initial velocity $\alpha_i$ for difficult tasks ($\lambda^* > \alpha_i$). No positive allocation can generate enough marginal surplus to justify crossing the threshold, so such tasks are abandoned and remain before their \Strict{} threshold with indicator function $\mathbb{I}(\cdot) = 0$.

\section{Methodology}
\label{sec:methods}

In this section, we translate the theoretical optimal policy derived in Theorem~\ref{thm:lambert} into \CLEAR{}, a practical inference-time control algorithm. Unlike heuristic methods that rely on manual rules for truncation or abandonment, \CLEAR{} numerically solves for the shadow price $\lambda^*$ that equilibrates the aggregate demand of the batch with the global supply of tokens.

\subsection{Threshold Prediction and Parametric Scaling}
The core input to our Lambert W policy is the emergence threshold $\tau_i$, which captures the minimum computation required for query $s_i$ to enter the productive reasoning regime. We employ a lightweight predictor $f_\theta: \mathcal{S} \to \mathbb{R}_+$ based on \texttt{DeBERTa-v3-base} to estimate this value:
\begin{equation}
    \hat{\tau}_i = \exp(f_\theta(s_i)).
\end{equation}
To map these predictions into actionable utility curves, \CLEAR{} relies on two critical parameters that govern the system's behavior:

\textbf{Initial Velocity $\alpha$:}
This parameter defines the initial magnitude of the utility function's derivative at the predicted threshold $\hat{\tau}$. 
It functions as the system's \textit{reservation price}: any query where the global market-clearing price $\lambda^*$ exceeds $\alpha$ yields a negative net surplus and is assigned zero tokens. 

\textbf{Decay Rate $\beta$:} 
This parameter governs the exponential rate at which marginal utility diminishes beyond the predicted length $\hat{\tau}$. It is inversely proportional to the characteristic length of the token allocation window. To accommodate varying traffic conditions, we implement an adaptive mechanism where $\beta$ is derived dynamically from the aggregate budget surplus:
\begin{equation}
    \beta = \frac{1}{\max(\epsilon, \bar{B} - \bar{\tau})}.
\end{equation}
Here, $\bar{B}$ is the designated budget constraint and $\bar{\tau}$ is the mean predicted budget. 

\textit{Remarks.} 
While Theorem~\ref{thm:lambert} provides a general solution for task-specific shaping parameters $\{\alpha_i, \beta_i\}$, predicting $\beta_i$ of an unseen query is computationally intractable and prone to high variance. 
Therefore, we adopt a model-intrinsic assumption: we treat $\alpha$ and $\beta$ as global hyperparameters that characterize the average reasoning dynamics of the LLM backbone itself, rather than the specific query. We assume task heterogeneity is primarily captured by the threshold parameter $\tau_i$, which we predict dynamically. 

\subsection{Global Shadow Price Optimization}

The core of \CLEAR{} lies in identifying the precise $\lambda^*$ that saturates the global budget constraint. Leveraging the monotonicity established in Section~\ref{sec:phase_transition}, where the individual optimal allocation $t_i^*(\lambda)$ strictly decreases with respect to the shadow price, we guarantee that the aggregate consumption function $C(\lambda) = \sum t_i^*(\lambda)$ is invertible. Therefore, we solve $\lambda^*$ using the bisection method, proceeding in three distinct steps:
\begin{itemize}
    \item \textbf{Search Space Initialization:} 
    We define the search interval $[\lambda_{\min}, \lambda_{\max}]$. The lower bound $\lambda_{\min} = 0$ represents the saturation regime (infinite budget). The upper bound $\lambda_{\max} = \alpha$ corresponds to the total abandonment regime. Crucially, $\alpha$ serves as the theoretical ceiling; any price $\lambda > \alpha$ yields negative surplus for all tasks, resulting in zero allocation.

    \item \textbf{Candidate Evaluation:} 
    In each iteration, we propose a candidate price $\lambda_{\text{mid}}$ and compute the optimal allocation for every query using the Lambert W Policy (Eq.~\ref{eq:lambert_solution}). This involves three substeps: first, computing the unconstrained surge length via the principal branch $W_0$ using the adaptive $\beta$; second, enforcing the solvency condition by zeroing out allocations where the cost $\lambda \cdot t$ exceeds the utility (Rational Abandonment); and third, applying the physical context limit $T_{\text{max}}$.

    \item \textbf{Market Clearing Update:} 
    We aggregate the individual allocations to determine the total consumption $C_{\text{total}}$. If $C_{\text{total}}$ exceeds the global budget $B_{\text{total}}$, indicating excess demand, we raise the price by setting $\lambda_{\min} \leftarrow \lambda_{\text{mid}}$. Conversely, if there is a supply surplus, we lower the price by setting $\lambda_{\max} \leftarrow \lambda_{\text{mid}}$.
\end{itemize}

The full procedure is summarized in Algorithm~\ref{alg:clear_solver}. Note that the Lambert W function $W_0(z)$ is efficiently computed via standard numerical libraries, thus introducing negligible latency compared to the autoregressive generation process.

\section{Experimental Settings}
\label{sec:experiments}

\paragraph{Models and Training Data}
We employ \texttt{Qwen2.5-Math-7B-Instruct} and \texttt{Qwen3-30B-A3B-Instruct}~\citep{yang2025qwen3} as the frozen backbones for all reasoning tasks. To estimate the latent emergence threshold $\tau(s)$, we fine-tune a \texttt{DeBERTa-v3-base} encoder as the threshold predictor. The predictor is trained to regress the logarithmic token length of the model's generated solutions, which serves as a proxy for the emergence threshold. The training data is drawn exclusively from splits of \texttt{GSM8K} and \texttt{MATH}.

\paragraph{Evaluation Datasets and Streams}
As illustrated in Figure~\ref{fig:traffic_stream_length}, we evaluate on a mixed oracle pool spanning \texttt{MATH-500}, \texttt{AMC-23}, \texttt{AIME-24}, \texttt{AIME-25}, \texttt{Minerva}, and \texttt{OlympiadBench}, and construct four synthetic inference streams: \texttt{Balanced}, \texttt{Mostly-Easy}, \texttt{Mostly-Hard}, and \texttt{U-Shaped}. The detailed training and evaluation setting is shown in Table~\ref{tab:dataset_stats} and \ref{tab:dataset_stats_30b}.

\begin{figure*}[htbp]
    \centering
    \includegraphics[width=1.0\linewidth]{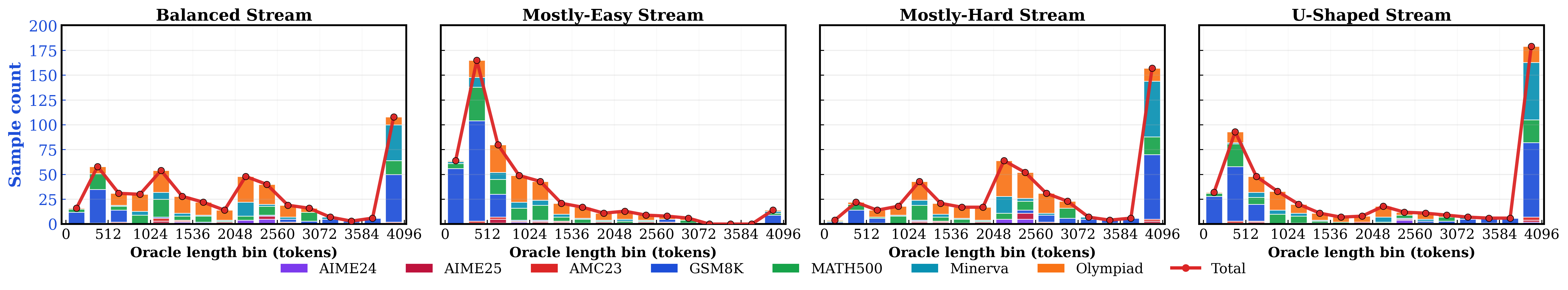}
    \caption{\textbf{Oracle-Length Distribution across Evaluation Traffic Streams.}
    Each panel shows one synthetic traffic stream (\texttt{Balanced}, \texttt{Mostly-Easy}, \texttt{Mostly-Hard}, and \texttt{U-Shaped}), with $n{=}500$ queries sampled from the 7B oracle pool.}
    \label{fig:traffic_stream_length}
\end{figure*}

\paragraph{Baselines and Allocation Policies}
We benchmark \CLEAR{} (\texttt{Lambert}) \footnote{We use \CLEAR{} (\texttt{Lambert}) to denote the full \CLEAR{} pipeline proposed above, distinguishing it from the internal \CLEAR{} variants used for ablation.} against several allocation strategies. \Uniform{} evenly distributes the global budget across all queries, while \Predictor{} assigns tokens proportional to the predicted threshold $\hat{\tau}_i$ but does not support abandonment. \TALEEP{} \citep{han2025token} serves as an external length-prediction baseline that estimates per-query token demand and renormalizes allocations to the same global budget. We also compare against two internal ablations: \CLEAR{} (\texttt{Heuristic}), which applies an affine allocation with a median-based rejection cutoff, and \CLEAR{} (\texttt{Auction}), which greedily admits queries by predicted return on investment under the budget constraint. Finally, \Oracle{} provides an upper bound using ground-truth solution lengths. Detailed formulations are provided in Appendix~\ref{app:allocation_details}.

\section{Results and Analysis}
\label{sec:results}

\subsection{Latent Threshold Modeling}
Accurately predicting the exact token consumption for reasoning tasks is inherently challenging. However, for our resource allocation purpose, the global shadow price $\lambda$ naturally acts as a normalizer, absorbing absolute prediction errors. Thus, the allocation primarily relies on the predictor's ability to correctly rank tasks by latent threshold. Figure~\ref{fig:prediction_analysis} demonstrates the log-scale regression performance of our predictor. 
\begin{figure}[htbp]
    \centering
    \includegraphics[width=1.0\linewidth]{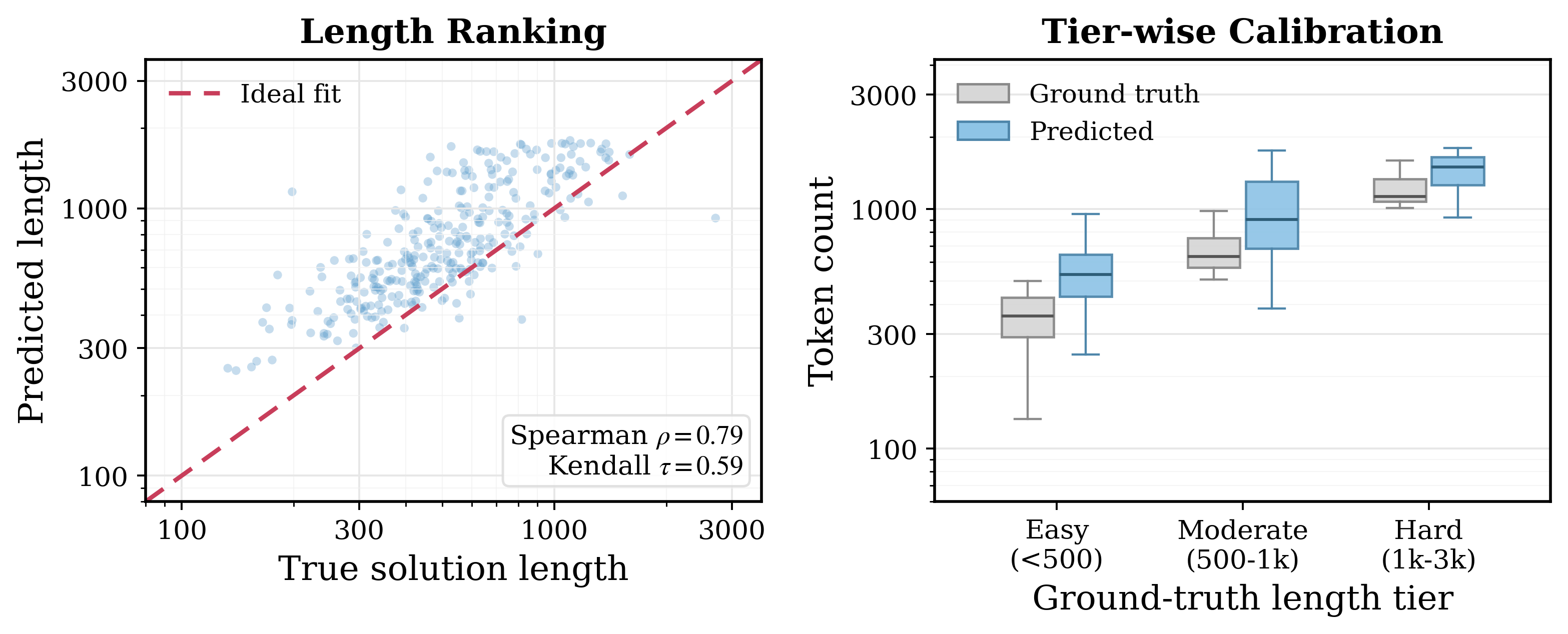}
    \caption{\textbf{Predictor Performance Analysis.} 
    \textbf{(Left)} Log-log scatter plot of predicted vs. actual token lengths. High rank correlation coefficients (Spearman's $\rho$, Kendall's $\tau$) indicate the model effectively captures relative threshold. 
    \textbf{(Right)} Stratified comparison of token count distributions across threshold tiers.}
    \label{fig:prediction_analysis}
\end{figure}

\begin{table*}[htbp]
    \centering
    \caption{\textbf{Comprehensive Performance Comparison across Evaluation Streams.} 
    We compare \CLEAR{} (\texttt{Lambert}) against various baselines including \Uniform{} allocation, Proportional \Predictor{}, and two \CLEAR{} variants. The \Oracle{} represents the theoretical upper bound. 
    \textbf{Bold} indicates the best practical performance among non-oracle methods. 
    The bottom row highlights the \textbf{Absolute Accuracy Gain} of \CLEAR{} (\texttt{Lambert}) over the \Uniform{}.}
    \label{tab:main_results}
    \vspace{2mm}
    \setlength{\tabcolsep}{3.5pt} 
    \renewcommand{\arraystretch}{1.2} 
    \begin{small}
    \begin{tabular}{l|cccc|cccc|cccc|cccc}
        \toprule
        \multirow{2}{*}{\textbf{Method}} & \multicolumn{4}{c|}{\textbf{Balanced Stream}} & \multicolumn{4}{c|}{\textbf{Mostly-Easy Stream}} & \multicolumn{4}{c|}{\textbf{Mostly-Hard Stream}} & \multicolumn{4}{c}{\textbf{U-Shaped Stream}} \\
         & \textbf{256} & \textbf{512} & \textbf{1024} & \textbf{2048} & \textbf{256} & \textbf{512} & \textbf{1024} & \textbf{2048} & \textbf{256} & \textbf{512} & \textbf{1024} & \textbf{2048} & \textbf{256} & \textbf{512} & \textbf{1024} & \textbf{2048} \\
        \midrule
        \Uniform & 3.0 & 12.4 & 17.4 & 21.6 & 9.0 & 33.8 & 45.6 & \textbf{49.0} & 1.0 & 5.0 & 6.8 & 11.4 & 4.4 & 18.4 & 25.4 & 27.2 \\
        \Predictor & 0.6 & 4.6 & \textbf{18.8} & \textbf{22.8} & 1.8 & 17.6 & 46.0 & \textbf{49.0} & 0.0 & 1.6 & 8.2 & \textbf{12.8} & 0.8 & 9.0 & 26.4 & 27.6 \\
        \TALEEP{}  & 3.8 & 6.8 & 16.0 & 22.2 & 7.2 & 10.8 & 22.4 & 39.8 & 4.2 & 5.8 & \textbf{9.2} & 12.0 & 7.2 & 10.4 & 24.6 & \textbf{29.0} \\
        \CLEAR{} (\texttt{Heaur.}) & 11.0 & \textbf{17.0} & \textbf{18.8} & \textbf{22.8} & 31.6 & 37.8 & 46.0 & \textbf{49.0} & 4.0 & \textbf{7.2} & 8.2 & \textbf{12.8} & 17.8 & 20.6 & 26.4 & 27.6 \\
        \CLEAR{} (\texttt{Auction}) & 14.2 & 15.6 & \textbf{18.8} & \textbf{22.8} & 32.8 & 39.2 & 46.0 & \textbf{49.0} & 6.0 & 6.6 & 8.2 & \textbf{12.8} & 18.4 & 22.2 & 26.4 & 27.6 \\
        \midrule
        \textbf{\CLEAR{} (\texttt{Lambert})} & \textbf{14.6} & 16.4 & \textbf{18.8} & \textbf{22.8} & \textbf{33.0} & \textbf{40.4} & \textbf{47.6} & \textbf{49.0} & \textbf{6.2} & 6.8 & 8.2 & \textbf{12.8} & \textbf{18.6} & \textbf{23.2} & \textbf{26.8} & 27.6 \\
        \midrule
        \Oracle & \color{gray}19.0 & \color{gray}21.6 & \color{gray}23.6 & \color{gray}26.0 & \color{gray}42.2 & \color{gray}48.4 & \color{gray}49.8 & \color{gray}49.8 & \color{gray}9.6 & \color{gray}11.4 & \color{gray}12.8 & \color{gray}16.6 & \color{gray}26.4 & \color{gray}27.6 & \color{gray}30.2 & \color{gray}32.8 \\
        \midrule
        \texttt{Gain (Abs.)} & \textcolor{green!60!black}{\textbf{+11.6}} & \textcolor{green!60!black}{+4.0} & \textcolor{green!60!black}{+1.4} & \textcolor{green!60!black}{+1.2} & \textcolor{green!60!black}{\textbf{+24.0}} & \textcolor{green!60!black}{+6.6} & \textcolor{green!60!black}{+2.0} & \textcolor{green!60!black}{+0.0} & \textcolor{green!60!black}{\textbf{+5.2}} & \textcolor{green!60!black}{+1.8} & \textcolor{green!60!black}{+1.4} & \textcolor{green!60!black}{+1.4} & \textcolor{green!60!black}{\textbf{+14.2}} & \textcolor{green!60!black}{+4.8} & \textcolor{green!60!black}{+1.4} & \textcolor{green!60!black}{+0.4} \\
        \bottomrule
    \end{tabular}
    \end{small}
\end{table*}

\begin{figure*}[htbp]
    \centering
    \includegraphics[width=1.0\linewidth]{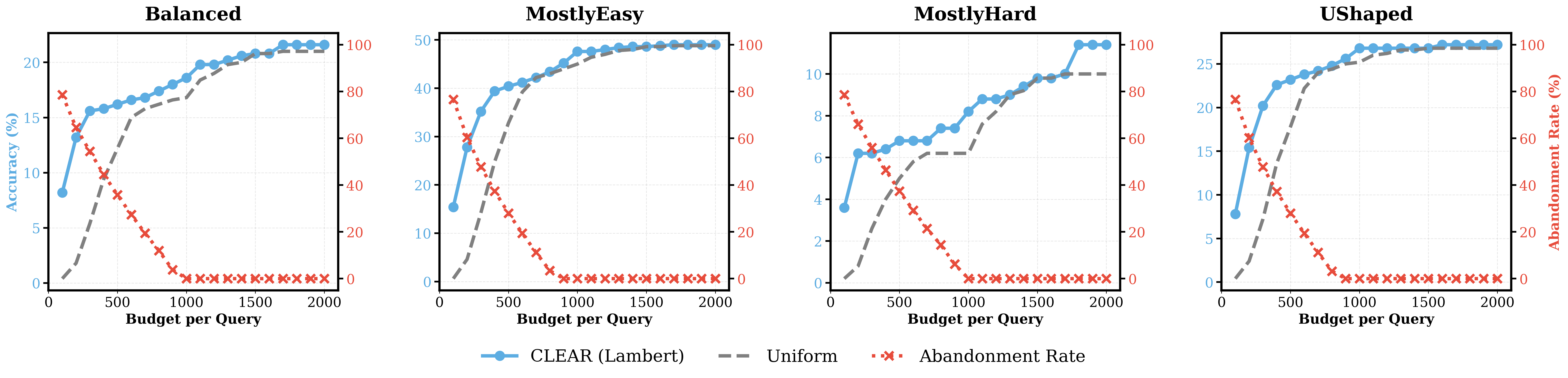}
    \vspace{-6mm}
    \caption{\textbf{Phase Transition Analysis.} 
    Dual-axis plots showing Accuracy in blue and Abandonment Rate in red as the global budget increases. \CLEAR{} significantly outperforms the gray dashed Uniform curve in the low-budget regime by maintaining a high abandonment rate. As the budget increases, abandonment drops to zero and the policies converge.}
    \label{fig:phase_transition}
\end{figure*}

\begin{figure*}[htbp]
    \centering
    \includegraphics[width=1.0\linewidth]{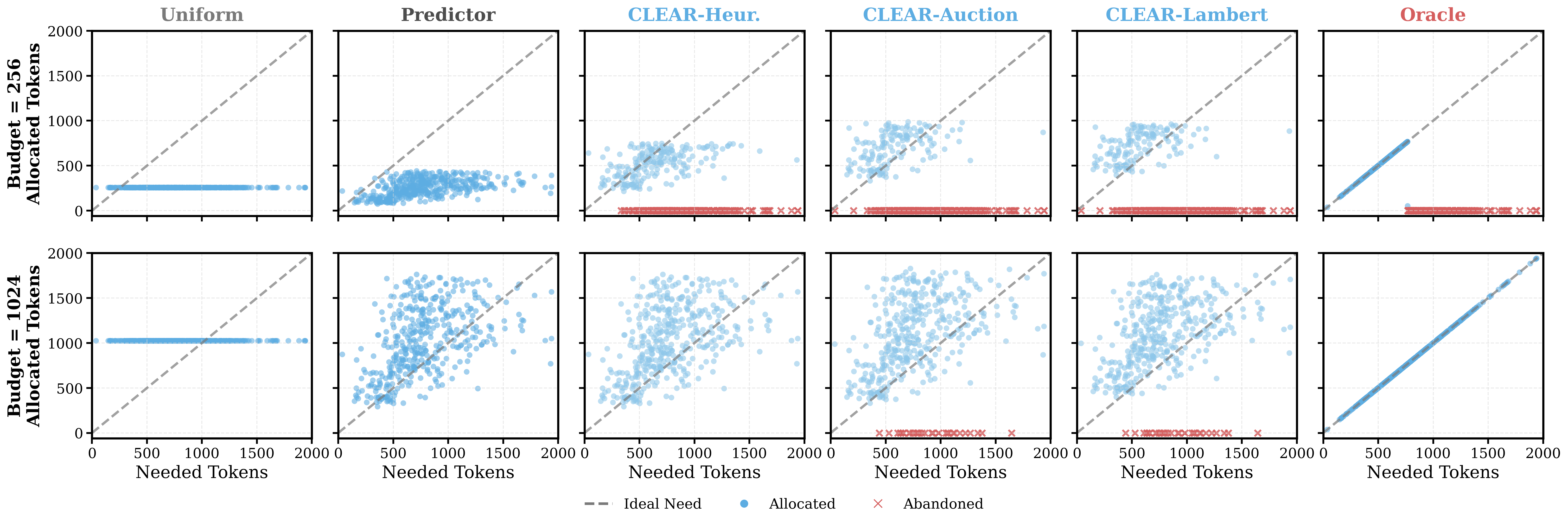}
    \caption{\textbf{Visualization of Allocation Policies.} 
    Each point represents a query from the Balanced stream. The top row shows the scarcity regime with a budget of 256, while the bottom row shows the abundance regime with a budget of 1024. Blue dots represent allocated tasks, and red crosses indicate abandoned tasks. }
    \label{fig:allocation_viz}
\end{figure*}

\begin{figure*}[htbp]
    \centering
    \includegraphics[width=1.0\linewidth]{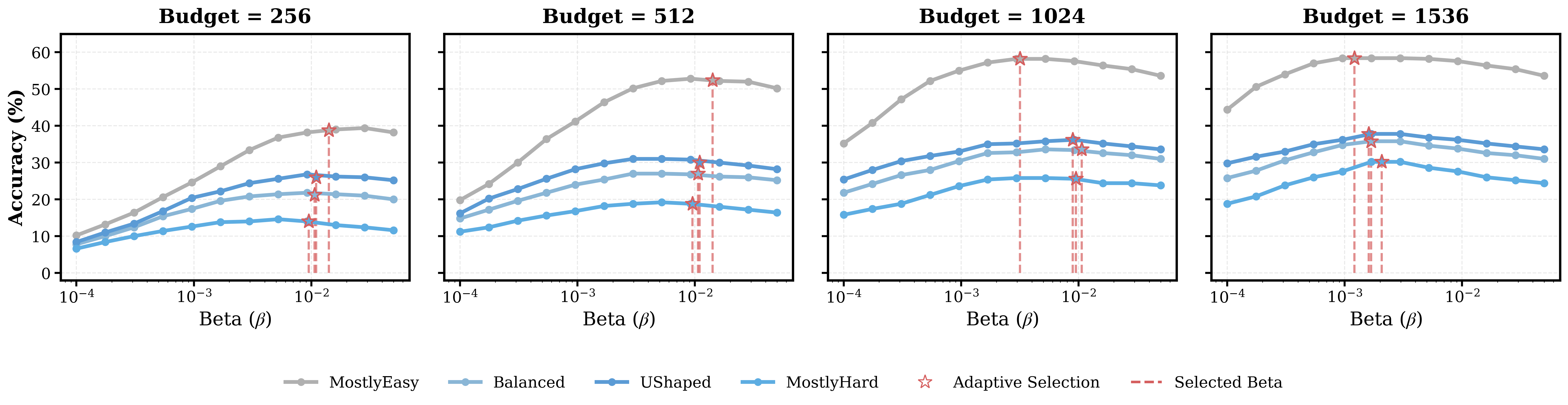} 
    \caption{\textbf{Sensitivity Analysis of the Decay Rate $\beta$.} The \textbf{star markers} ($\star$) indicate the operating points automatically selected by our adaptive $\beta$ mechanism.}
    \label{fig:beta_sensitivity}
\end{figure*}

\begin{figure*}[htbp]
    \centering
    \includegraphics[width=1.0\linewidth]{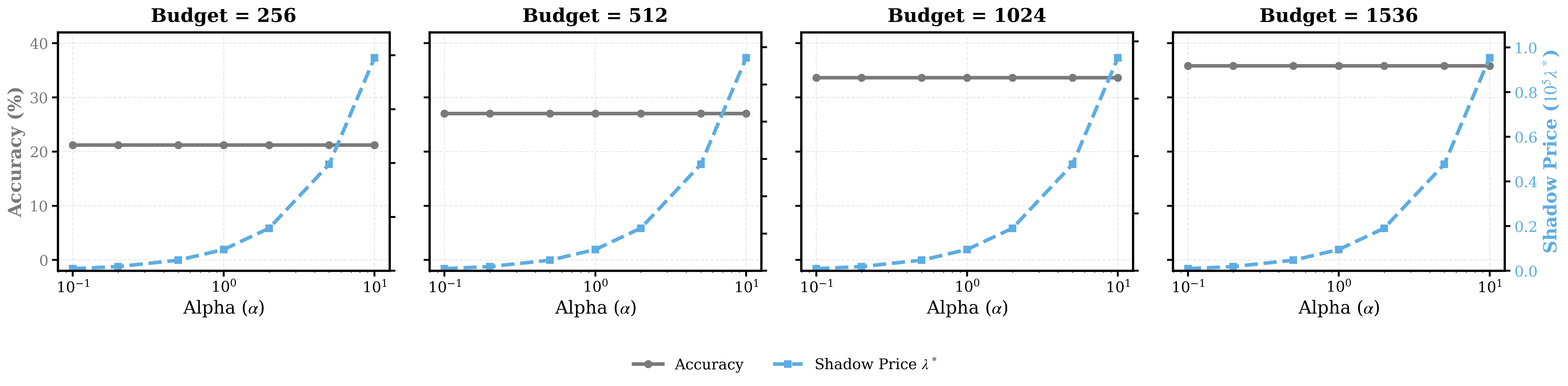} 
    \caption{\textbf{Scale Invariance of the Initial Velocity $\alpha$.} The left axis (grey) shows that accuracy remains effectively constant. The right axis (blue dashed) reveals the optimized shadow price $\lambda^*$ scales linearly with $\alpha$ ($\log \lambda^* \propto \log \alpha$).}
    \label{fig:alpha_invariance}
\end{figure*}

\subsection{Main Performance Comparison}

\paragraph{Comparison on Mathematical Reasoning.} Figure~\ref{fig:allocation_viz} and Table~\ref{tab:main_results} summarizes the accuracy of \CLEAR{} compared to the Uniform, Predictor-based, and Heuristic and Auction-based variants. The largest improvements occur under the most constrained budget of 256 tokens per query. In the \textbf{\texttt{Balanced}} stream, \CLEAR{} improves over the Uniform policy by \textcolor{green!60!black}{\textbf{+11.6}} accuracy points. In the \textbf{\texttt{Mostly-Easy}} stream, \CLEAR{} achieves an even larger gain of \textcolor{green!60!black}{\textbf{+24.0}} accuracy points, indicating that uniform allocation wastes scarce tokens on queries that cannot produce sufficient utility under the global budget. In the \textbf{\texttt{Mostly-Hard}} and \textbf{\texttt{U-Shaped}} streams, \CLEAR{} further improves over Uniform by \textcolor{green!60!black}{\textbf{+5.2}} and \textcolor{green!60!black}{\textbf{+14.2}} points, respectively.
As the budget increases, the absolute gains naturally shrink because most methods receive enough tokens to cross the latent emergence threshold. This trend confirms that \CLEAR{} is most beneficial in resource-scarce regimes, where shadow-price allocation and rational abandonment prevent the system from spreading compute too thinly across the batch.

\paragraph{Generalization to Code Generation}
To examine whether the proposed allocation principle extends beyond mathematical reasoning, we evaluate \CLEAR{} on code-generation tasks. We use \texttt{Qwen2.5-Coder-7B} and retrain the length regressor in-domain using completion lengths from \texttt{HumanEval} and \texttt{MBPP}, while keeping the allocation algorithm unchanged. Results are reported under a best-of-$K$ protocol with $K=4$: each trajectory has a cap of $B=1024$, and each query receives the same total completion-token budget of $4B$ across methods.

\begin{table}[htbp]
    \centering
    \caption{\textbf{Code Generation Results.} 
    Accuracy comparison under aligned best-of-4 completion-token budgets.}
    \label{tab:code_generation}
    \vspace{1mm}
    \setlength{\tabcolsep}{5pt}
    \renewcommand{\arraystretch}{1.15}
    \begin{small}
    \resizebox{\linewidth}{!}{
    \begin{tabular}{lccc}
        \toprule
        \textbf{Method} & \textbf{HumanEval+} & \textbf{MBPP+} & \textbf{BigCodeBench} \\
        \midrule
        \Uniform{} & 36.8 & 39.4 & 16.2 \\
        \textbf{\CLEAR{} (\texttt{Lambert})} & \textbf{43.1} & \textbf{45.9} & \textbf{19.8} \\
        \midrule
        \texttt{Gain (Abs.)} & \textcolor{green!60!black}{\textbf{+6.3}} & \textcolor{green!60!black}{\textbf{+6.5}} & \textcolor{green!60!black}{\textbf{+3.6}} \\
        \bottomrule
    \end{tabular}
    }
    \end{small}
\end{table}

As shown in Table~\ref{tab:code_generation}, \CLEAR{} consistently improves over uniform allocation across all three code benchmarks. This suggests that global budget clearing and strategic abandonment are not limited to the math-focused setting, but can also benefit stochastic code generation under strict compute constraints.

\paragraph{Visualization of Token Allocation}
To understand the mechanism behind the performance of \CLEAR{}, we visualize the token allocation in Figure~\ref{fig:allocation_viz}.

Under the scarcity regime of 256 tokens, the \Uniform{} and \Predictor-based policies distribute the budget continuously but insufficiently across all queries. Consequently, the majority of tasks receive fewer tokens than their ground-truth reasoning length, leading to systemic failure. In contrast, \CLEAR{} explicitly abandons tasks where the computational cost outweighs the potential gain. Crucially, the resources conserved from these discarded tasks are redistributed to the admitted queries. Unlike binary cut-off or auction, the Lambert policy dynamically adjusts the extra allowance based on the decay rate $\beta$, ensuring valid tasks receive ample resources to handle uncertainty.

\subsection{Robustness and Sensitivity}
We further analyze the robustness of \CLEAR{} across three dimensions: phase transitions, predictor noise, and hyperparameter sensitivity.

\textbf{Robustness to Predictor Noise.}
Figure~\ref{fig:noise_sensitivity} shows the accuracy degradation as we inject log-normal noise into the predictor. While performance naturally declines with increasing noise, \CLEAR{} maintains a significant advantage over Uniform even at high noise levels.

\textbf{Sensitivity to Hyperparameters.}
We examine the system's sensitivity to its two core economic parameters: the decay rate $\beta$ and the initial velocity $\alpha$.

Figure~\ref{fig:beta_sensitivity} evaluates the impact of static $\beta$ values. We observe a distinct regime shift: low-budget scenarios require a strict, high $\beta$ to minimize waste and ensure solvency for easier tasks, whereas high-budget scenarios benefit from a lenient, low $\beta$ that allows for extended reasoning surges. Therefore, static $\beta$ cannot simultaneously satisfy these conflicting demands. Our adaptive mechanism dynamically calibrates $\beta$ based on the aggregate surplus, consistently landing on the Pareto-optimal frontier across all regimes without manual tuning.
Unlike $\beta$, the threshold $\alpha$ exhibits a unique scale-invariant property. As shown in Figure~\ref{fig:alpha_invariance}, varying $\alpha$ results in negligible fluctuations in accuracy. Meanwhile, the market-clearing shadow price $\lambda^*$ scales linearly with $\alpha$. Since the Lambert W allocation depends on the ratio $\lambda/\alpha$, the optimizer automatically adjusts $\lambda^*$ to offset any changes in $\alpha$. 

\subsection{Structural Utility Variants}
To test whether the gains of \CLEAR{} depend on the exact shifted-surge parameterization, we evaluate two alternative utility structures that preserve the same threshold-aware allocation principle but replace the latent utility shape. The \texttt{Triangular} variant uses a tent-shaped utility that rises linearly after the predicted threshold and then decays linearly after its peak. The \texttt{Quadratic} variant uses a concave quadratic peak over the post-threshold extension. In both cases, the system still solves for a global budget-clearing allocation and applies rational abandonment under tight budgets.
Let $\Delta t = t - \tau_i$ denote the effective post-threshold generation length, $c_i$ denote the maximum supported post-threshold extension, $p_i \in (0,c_i)$ denote the triangular peak location, and $u_i^{\max}$ denote the peak utility. The \texttt{Triangular} utility is defined as:
\begin{equation}
    \phi_i^{\texttt{Tri}}(t) =
    \begin{cases}
    0 & 0 \le t < \tau_i \\
    u_i^{\max} \cdot \frac{\Delta t}{p_i} & 0 \le \Delta t \le p_i \\
    u_i^{\max} \cdot \frac{c_i - \Delta t}{c_i - p_i} & p_i < \Delta t \le c_i \\
    0 & \Delta t > c_i .
    \end{cases}
    \label{eq:triangular_potential}
\end{equation}

The \texttt{Quadratic} utility uses a smooth concave peak over the same post-threshold support:
\begin{equation}
    \phi_i^{\texttt{Quad}}(t) =
    \begin{cases}
    0 & 0 \le t < \tau_i \\
    u_i^{\max} \cdot \frac{4\Delta t(c_i-\Delta t)}{c_i^2} & 0 \le \Delta t \le c_i \\
    0 & \Delta t > c_i .
    \end{cases}
    \label{eq:quadratic_potential}
\end{equation}

\begin{table}[htbp]
    \centering
    \caption{\textbf{Structural Utility Variant Results.}
    Accuracy comparison under the tight-budget setting ($\bar{B}=256$).}
    \label{tab:structural_variants}
    \vspace{1mm}
    \setlength{\tabcolsep}{4pt}
    \renewcommand{\arraystretch}{1.15}
    \begin{small}
    \resizebox{\linewidth}{!}{
    \begin{tabular}{lccc}
        \toprule
        \textbf{Method} & \textbf{Balanced} & \textbf{Mostly-Hard} & \textbf{U-Shaped} \\
        \midrule
        \Uniform{} & 5.4 & 2.2 & 8.2 \\
        \CLEAR{} (\texttt{Lambert}) & \textbf{19.4} & \textbf{13.4} & \textbf{23.8} \\
        \textbf{\texttt{Triangular}} (\texttt{Lambert}) & 18.8 & 13.0 & 21.6 \\
       \textbf{ \texttt{Quadratic}} (\texttt{Lambert}) & \textbf{19.4} & \textbf{13.4} & 23.6 \\
        \bottomrule
    \end{tabular}
    }
    \end{small}
\end{table}

\begin{figure}[htbp]
    \centering
    \includegraphics[width=1.0\linewidth]{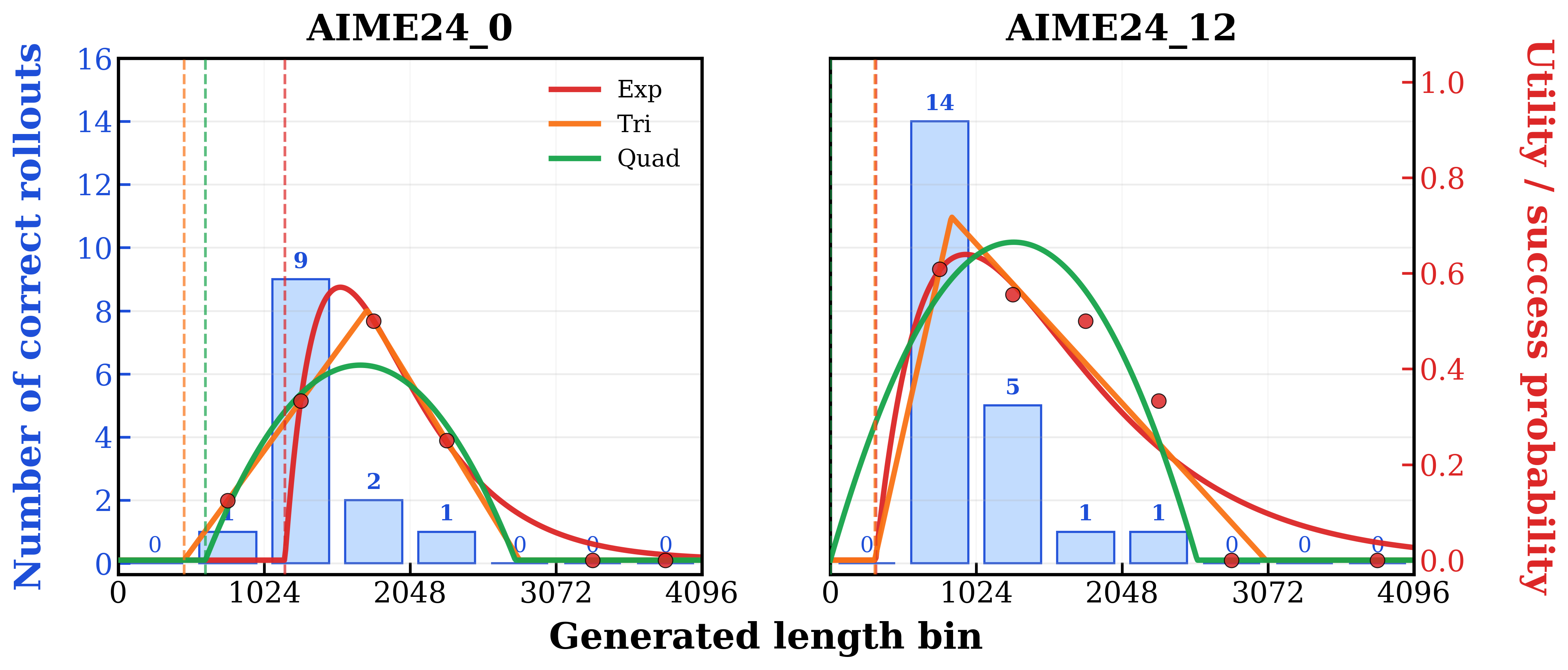}
    \caption{\textbf{Visualization of Structural Utility Variants.}}
    \label{fig:structural_variants}
\end{figure}

Table~\ref{tab:structural_variants} shows that both structural variants remain substantially above uniform allocation in the most resource-constrained regime. This suggests that the main benefit comes from global budget clearing and abandonment-aware allocation, rather than from a single fragile choice of utility curve.


\section{Related Work}
\paragraph{Resource Allocation for LLM Inference.}
LLM deployment is increasingly studied as a resource allocation problem, where computation, latency, and tokens are scarce assets to be assigned under global constraints. Systems such as FrugalGPT route queries through model cascades to reduce cost \citep{chen2023frugalgpt}, while RouteLLM learns preference-based routing policies across models \citep{ong2024routellm}. Other work formulates model selection or query allocation as contextual bandits, knapsack optimization, or cascade scheduling, including LLM-Bandit and TREACLE \citep{zhang2024efficient}. Complementary infrastructure analyses characterize the hardware-level Pareto frontier of throughput, latency, and energy \citep{kwon2023efficient, sheng2023flexgen, wilhelm2025beyond}. 

\paragraph{Efficient CoT.}
CoT prompting and inference-time scaling show that allocating more computation at test time can substantially improve reasoning performance \citep{wei2022chain, snell2024scaling, wu2024inference, wan2025adapthink}. A related line of work improves token efficiency by adapting the reasoning procedure itself. DSC \citep{wang2025make} allocates different numbers of sampled reasoning paths according to question difficulty and posterior confidence. SelfBudgeter \citep{li2025selfbudgeter} trains the model to estimate and obey instance-specific reasoning budgets, reducing unnecessary overthinking. TALE \citep{han2025token} estimates a token budget for each problem and injects it into the prompt to guide shorter CoT generation. These methods focus on controlling or aggregating reasoning traces within a query. In contrast, \CLEAR{} addresses a batch-level allocation problem: given fixed decoding behavior and a fixed global token budget, it decides how much budget each query should receive and when a query should be rationally abandoned.

\section{Conclusion}
\label{sec:conclusion}
In this work, we study how to allocate a fixed inference-time token budget across heterogeneous reasoning queries. We show that uniform budgeting can waste tokens on hard queries that cannot be solved within the available budget, while overlooking queries that could be solved with a more targeted allocation. To address this mismatch, we formulate batch-level token allocation as a constrained optimization problem and introduce \CLEAR{}. \CLEAR{} assigns zero budget to queries with negative expected surplus and reallocates tokens to queries where additional computation is more likely to improve utility. Extensive experiments show that \CLEAR{} improves the cost-accuracy Pareto frontier, with the largest gains appearing in resource-scarce regimes where uniform allocation spreads computation too thinly. Future work will focus on narrowing the gap toward oracle-level allocation by improving the robustness of utility modelling.

\section*{Impact Statement}
This paper presents work whose goal is to advance the field
of LLM and machine learning. There are many potential
societal consequences of our work, none which we feel must
be specifically highlighted here.

\section*{Acknowledgements}
This work was supported by the National Natural
Science Foundation of China under Grant 72571007 and Grant 72595830/72595831.

\bibliography{main}
\bibliographystyle{icml2026}

\appendix
\section{Appendix: Algorithm Details}
\begin{algorithm}[ht]
\caption{\CLEAR{} (\texttt{Lambert})}
\label{alg:clear_solver}
\begin{algorithmic}[1]
\STATE \textbf{Input:} Batch $\mathcal{S}=\{s_1,\dots,s_N\}$, total budget $B_{\text{total}}$, token cap $T_{\max}$
\STATE \textbf{Params:} Initial velocity $\alpha$, tolerance $\epsilon$
\STATE \textbf{Output:} Allocation vector $\mathbf{t}^*$

\STATE $\hat{\boldsymbol{\tau}} \leftarrow \text{PredictThreshold}(\mathcal{S})$
\STATE $\bar{B} \leftarrow B_{\text{total}}/N,\quad \bar{\tau} \leftarrow \frac{1}{N}\sum_{i=1}^N \hat{\tau}_i$
\STATE $\beta \leftarrow 1 / \max(\epsilon, \bar{B}-\bar{\tau})$ \COMMENT{batch-adaptive decay rate}
\STATE $\mathbf{t}^{\text{best}} \leftarrow \mathbf{0}$

\STATE \textit{// Bisection search for the shadow price $\lambda^*$}
\STATE $\lambda_{\min} \leftarrow 0, \quad \lambda_{\max} \leftarrow \alpha - 10^{-6}$

\WHILE{$(\lambda_{\max} - \lambda_{\min}) > \epsilon$}
    \STATE $\lambda \leftarrow (\lambda_{\min} + \lambda_{\max}) / 2$
    \STATE $C_{\text{current}} \leftarrow 0$
    
    \FOR{$i=1$ to $N$}
        \IF{$\lambda \ge \alpha$} 
            \STATE $t_i \leftarrow 0$
        \ELSE
            \STATE $z \leftarrow (\lambda e) / \alpha$
            \STATE $u \leftarrow W_0(z)$
            \STATE $\Delta t_i \leftarrow (1-u)/\beta$
            \IF{$\Delta t_i \le 0$}
                \STATE $t_i \leftarrow 0$
            \ELSE
                \STATE $t_i^{\text{unc}} \leftarrow \hat{\tau}_i + \Delta t_i$
                \STATE $\phi_i \leftarrow \alpha \Delta t_i \exp(-\beta \Delta t_i)$
                \IF{$\phi_i > \lambda t_i^{\text{unc}}$}
                    \STATE $t_i \leftarrow \min(t_i^{\text{unc}}, T_{\max})$
                \ELSE
                    \STATE $t_i \leftarrow 0$
                \ENDIF
            \ENDIF
        \ENDIF
        \STATE $C_{\text{current}} \leftarrow C_{\text{current}} + t_i$
    \ENDFOR
    
    \IF{$C_{\text{current}} > B_{\text{total}}$}
        \STATE $\lambda_{\min} \leftarrow \lambda$
    \ELSE
        \STATE $\lambda_{\max} \leftarrow \lambda$
        \STATE $\mathbf{t}^{\text{best}} \leftarrow \mathbf{t}$
    \ENDIF
\ENDWHILE

\STATE $\mathbf{t}^* \leftarrow \lfloor \mathbf{t}^{\text{best}} \rfloor$
\STATE Clip each $t_i^*$ to $[0,T_{\max}]$ and optionally correct residual tokens.
\STATE \textbf{return} $\mathbf{t}^*$
\end{algorithmic}
\end{algorithm}

\subsection{Pseudocode}
We present the pseudocode for the \CLEAR{} in Algorithm \ref{alg:clear_lambert}, which computes the optimal resource allocation vector under a given total budget constraint.

\subsection{Allocation Policy Details}
\label{app:allocation_details}

In this section, we formally define the token allocation $t_i$ for a batch of queries $\mathcal{S} = \{s_1, \dots, s_N\}$ under a total budget constraint $B_{\text{total}}$. Let $\hat{\tau}_i$ denote the predicted emergence threshold used by practical policies, and let $d_i$ denote the ground-truth solution length used only by the \Oracle{} policy. The average budget per query is $\bar{B} = B_{\text{total}} / N$.

\subsubsection{\Uniform{} Policy}
This is the standard baseline where resources are distributed equally, agnostic to task complexity.
\begin{equation}
    t_i^{\text{unif}} = \lfloor \bar{B} \rfloor.
\end{equation}

\subsubsection{\Predictor{}-Based Proportional Policy}
This policy assumes a linear relationship between the latent threshold and required tokens, scaling the budget proportionally to the predicted length without abandonment.
\begin{equation}
    t_i^{\text{prop}} = \frac{\hat{\tau}_i}{\sum_{j=1}^N \hat{\tau}_j} \cdot B_{\text{total}}.
\end{equation}
In practice, we clip the allocation to a minimum (e.g., 32 tokens) and maximum context length.

\subsubsection{\TALEEP{} Policy}
\TALEEP{} first uses a stronger teacher model to estimate per-query completion needs, then enforces both a soft prompt-level budget constraint and a hard decoding cap during regeneration.
Let $\tilde{\tau}_i$ denote the raw budget estimate from the teacher model (e.g., \texttt{Qwen3-30B-Instruct}).

\begin{enumerate}
    \item \textbf{Teacher-based budget estimation:}
    For each query $s_i$, the teacher predicts a raw token need $\tilde{\tau}_i$.

    \item \textbf{Renormalized proportional allocation:}
    The raw estimates are converted into per-query caps under the same total budget:
    \begin{equation}
        \hat{t}_i^{\text{tale}} = \frac{\tilde{\tau}_i}{\sum_{j=1}^N \tilde{\tau}_j} \cdot B_{\text{total}}.
    \end{equation}
    In practice, we apply clipping and integerization:
    \begin{equation}
        t_i^{\text{tale}} = \mathrm{clip}\!\left(\left\lfloor \hat{t}_i^{\text{tale}} \right\rceil,\, t_{\min},\, t_{\max}\right),
    \end{equation}
    followed by a residual correction step to ensure exact budget conservation:
    \begin{equation}
        \sum_{i=1}^N t_i^{\text{tale}} = B_{\text{total}}.
    \end{equation}

    \item \textbf{Budget-conditioned regeneration:}
    For each query $s_i$, we construct a budget-aware prompt with a soft instruction (e.g., “keep the entire response within at most $t_i^{\text{tale}}$ completion tokens”), and run decoding with a hard cap:
    \begin{equation}
        y_i \sim p_\theta(\cdot \mid s_i,\; t_i^{\text{tale}}), 
        \qquad \texttt{max\_tokens}=t_i^{\text{tale}}.
    \end{equation}
    We use greedy decoding in our setup ($\texttt{temperature}=0$, $\texttt{top\_p}=1$).
\end{enumerate}

Final accuracy is computed by applying \texttt{math\_verify} to regenerated outputs $\{y_i\}_{i=1}^N$.
\subsubsection{\CLEAR{}(\texttt{Heuristic}) Policy}
This baseline uses a fixed cutoff rule based on the predicted thresholds. When the average budget is too small relative to the average predicted threshold, it drops the harder half of the batch and reallocates the budget to the remaining queries.
\begin{enumerate}
    \item \textbf{Check budget scarcity:} Let $\hat{\mathcal{T}}=\{\hat{\tau}_1,\dots,\hat{\tau}_N\}$ and define $\eta = \bar{B} / \mathbb{E}[\hat{\mathcal{T}}]$. If $\eta < 0.8$, the cutoff rule is activated.
    \item \textbf{Select easier queries:} Queries with $\hat{\tau}_i > \text{Median}(\hat{\mathcal{T}})$ receive zero budget. The remaining queries share the total budget using:
    \begin{equation}
        t_i^{\text{heur}} = 
        \begin{cases} 
        \mu_{\text{sur}} + \kappa \cdot (\hat{\tau}_i - \mu_{\text{sur}}) & \text{if } \hat{\tau}_i \le \text{Median}(\hat{\mathcal{T}}) \\
        0 & \text{otherwise,}
        \end{cases}
    \end{equation}
    where $\mu_{\text{sur}}$ is the mean predicted threshold among selected queries, and $\kappa$ is chosen so that the selected queries use exactly $B_{\text{total}}$ tokens in total.
\end{enumerate}

\subsubsection{\CLEAR{}(\texttt{Auction}) Policy}
This ablation uses a simple greedy selection rule. It first keeps the queries with smaller predicted thresholds, since they are more likely to be completed under a limited budget, and then distributes the available tokens among the selected queries.
\begin{enumerate}
    \item \textbf{Sort by predicted length:} Sort the query indices by increasing predicted threshold, yielding a permutation $(p_1,\dots,p_N)$ such that $\hat{\tau}_{p_1} \le \dots \le \hat{\tau}_{p_N}$.
    \item \textbf{Select survivors:} Keep the largest prefix of this sorted list whose predicted thresholds fit within the total budget:
    \begin{equation}
        m^* = \max \left\{ m \in [1, N] \mid \sum_{j=1}^m \hat{\tau}_{p_j} \le B_{\text{total}} \right\}.
    \end{equation}
    \item \textbf{Allocate to survivors:} Assign zero budget to all non-selected queries. The selected queries then share the full budget using the same affine allocation rule as \CLEAR{} (\texttt{Heuristic}), rescaled so that $\sum_i t_i = B_{\text{total}}$.
\end{enumerate}

\subsubsection{\Oracle{} Policy}
This policy is an upper-bound baseline that uses ground-truth solution lengths $d_i$, which are unavailable at test time. It spends tokens on the shortest solvable queries first, maximizing the number of completed tasks under the total budget.
\begin{enumerate}
    \item \textbf{Sort by true length:} Sort indices into a permutation $(o_1,\dots,o_N)$ such that $d_{o_1} \le \dots \le d_{o_N}$.
    \item \textbf{Fill the budget greedily:} Allocate exactly $d_{o_j}$ tokens to each query in sorted order until the next query would exceed $B_{\text{total}}$:
    \begin{equation}
        t_{o_j}^{\text{oracle}} = 
        \begin{cases} 
        d_{o_j} & \text{if } \sum_{l=1}^j d_{o_l} \le B_{\text{total}} \\
        0 & \text{otherwise.}
        \end{cases}
    \end{equation}
\end{enumerate}

\section{Appendix: More Results}
\label{app:results}

To demonstrate the generalizability of our framework to larger-scale reasoning models, we present additional experimental results using Qwen3-30B-A3B-Instruct. 

Figure \ref{fig:prediction_scatter_30b} validates the threshold predictor's accuracy on this larger model, while Table \ref{tab:main_results_30b} details the downstream allocation performance across varying supply-demand scenarios. In summary, CLEAR consistently outperforms the uniform baseline across all tested conditions, achieving up to \textbf{+2.4} accuracy gains in strictly constrained budget environments where efficient resource distribution is most critical.

\begin{figure}[htbp]
    \centering
    \includegraphics[width=1.0\linewidth]{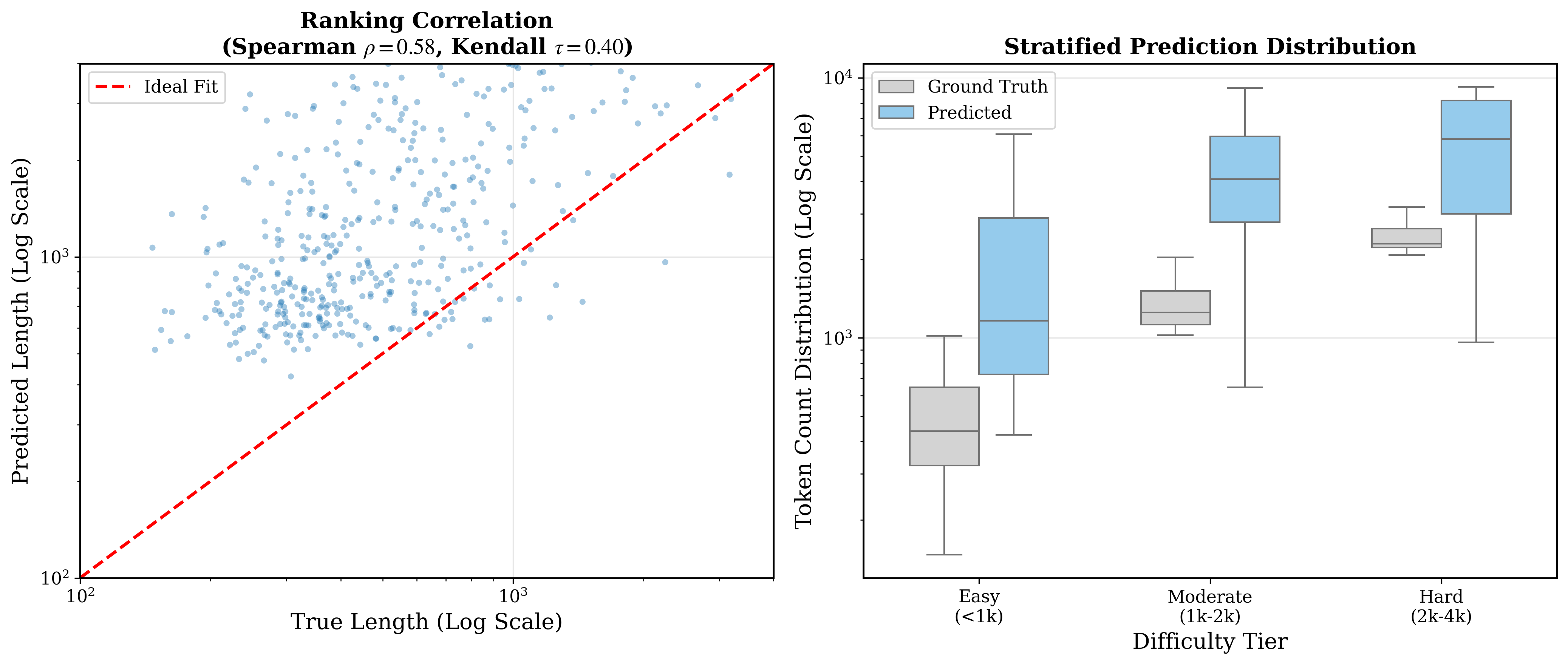}
    \caption{\textbf{Predictor Performance Analysis on Qwen3-30B-A3B-Instruct.} }
    \label{fig:prediction_scatter_30b}
\end{figure}

\begin{table*}[htbp]
    \centering
    \caption{\textbf{Performance Comparison across Difficulty Scenarios and Budgets on Qwen3-30B-A3B-Instruct.} 
    We compare our proposed \textbf{CLEAR} against Uniform, Direct Predictor, and heuristic CLEAR variants. 
    Results are reported as accuracy (\%).
    \textbf{Bold} indicates the best performance among non-oracle methods. 
    The bottom row shows the absolute accuracy gain of CLEAR over the Uniform baseline.}
    \label{tab:main_results_30b}
    \vspace{2mm}
    \setlength{\tabcolsep}{5pt} 
    \renewcommand{\arraystretch}{1} 
    \begin{small}
    \begin{tabular}{l|cc|cc|cc|cc}
        \toprule
        \multirow{2}{*}{\textbf{Method}} & \multicolumn{2}{c|}{\textbf{Balanced Stream}} & \multicolumn{2}{c|}{\textbf{Mostly-Easy}} & \multicolumn{2}{c|}{\textbf{Mostly-Hard}} & \multicolumn{2}{c}{\textbf{U-Shaped}} \\
         & \textbf{1024} & \textbf{4096} & \textbf{1024} & \textbf{4096} & \textbf{1024} & \textbf{4096} & \textbf{1024} & \textbf{4096} \\
        \midrule
        Uniform & 16.4 & 26.4 & 46.6 & 54.4 & 6.4 & 14.2 & 26.6 & 30.0 \\
        Predictor & 8.0 & \textbf{28.0} & 32.4 & \textbf{55.6} & 1.8 & \textbf{15.2} & 13.4 & \textbf{31.6} \\
        CLEAR (Heur.) & 18.4 & \textbf{28.0} & 37.2 & \textbf{55.6} & \textbf{9.8} & \textbf{15.2} & 23.4 & \textbf{31.6} \\
        CLEAR (Auction) & 16.2 & 24.8 & 41.0 & 54.2 & 7.0 & 13.0 & 24.4 & 29.6 \\
        \midrule
        \textbf{CLEAR (Ours)} & \textbf{18.8} & 26.8 & \textbf{47.4} & 54.6 & 8.6 & 14.8 & \textbf{27.2} & 30.6 \\
        \midrule
        Oracle & \color{gray}28.8 & \color{gray}38.0 & \color{gray}55.4 & \color{gray}57.8 & \color{gray}17.4 & \color{gray}27.8 & \color{gray}33.2 & \color{gray}38.0 \\
        \midrule
        \textit{Gain (vs. Uniform)} & \textcolor{green!60!black}{\textbf{+2.4}} & \textcolor{green!60!black}{+0.4} & \textcolor{green!60!black}{\textbf{+0.8}} & \textcolor{green!60!black}{+0.2} & \textcolor{green!60!black}{\textbf{+2.2}} & \textcolor{green!60!black}{+0.6} & \textcolor{green!60!black}{\textbf{+0.6}} & \textcolor{green!60!black}{+0.6} \\
        \bottomrule
    \end{tabular}
    \end{small}
\end{table*}

\begin{figure*}[htbp]
    \centering
    \includegraphics[width=1.0\linewidth]{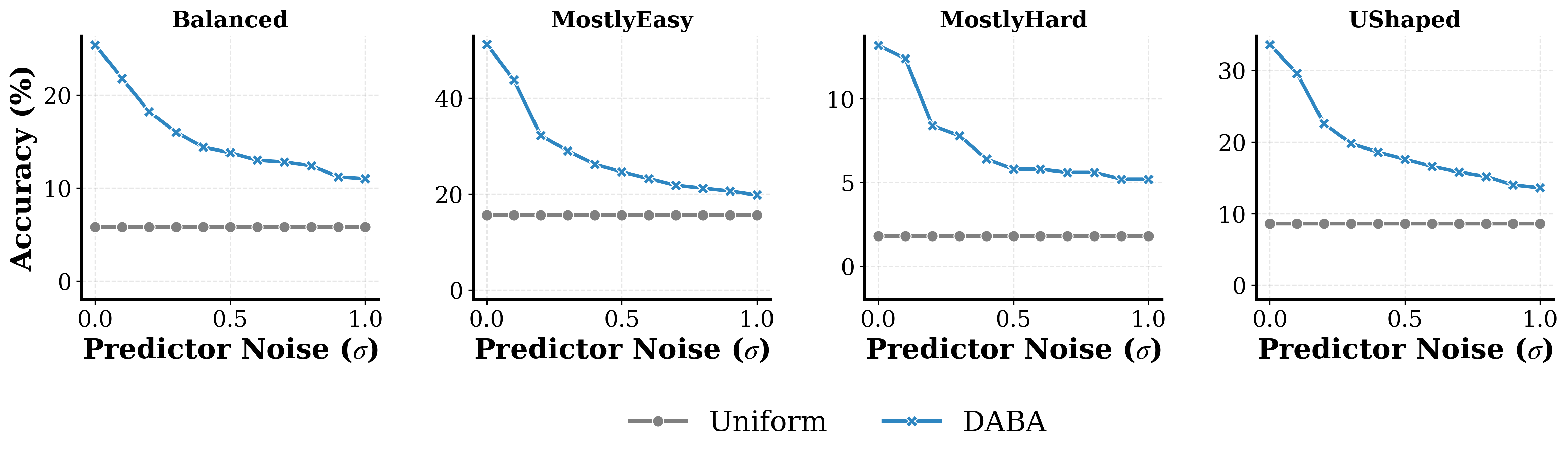}
    \caption{\textbf{Robustness to Predictor Noise.} 
    Performance of CLEAR versus Uniform under increasing predictor noise $\sigma$. CLEAR maintains a performance advantage over the baseline even under significant noise levels.}
    \label{fig:noise_sensitivity}
\end{figure*}

\begin{table*}[htbp]
    \centering
    \caption{\textbf{Detailed Experimental Configuration.} This table summarizes the hyperparameters and settings used for data generation, predictor training, and the CLEAR allocation mechanism across all experiments.}
    \label{tab:experiment_config}
    \vspace{2mm}
    \renewcommand{\arraystretch}{1.2}
    \begin{small}
    \begin{tabular}{p{3.5cm} p{5.5cm} p{6.5cm}}
        \toprule
        \textbf{Component} & \textbf{Parameter / Setting} & \textbf{Value / Description} \\
        \midrule
        \multicolumn{3}{l}{\textit{\textbf{1. Data Generation (Oracle)}}} \\
        \midrule
        & Backbone Models & Qwen2.5-Math-7B, Qwen3-30B-A3B-Instruct \\
        & Max New Tokens & 16,384 (for 30B), 8,192 (for 7B) \\
        & Decoding Strategy & Greedy Decoding (Temperature = 0) \\
        \midrule
        \multicolumn{3}{l}{\textit{\textbf{2. Threshold Predictor}}} \\
        \midrule
        & Architecture & DeBERTa-v3-base (86M parameters) \\
        & Training Sources & GSM8K (Train), MATH (Train) \\
        & Input Tokenization & Left-Truncation (Retain last 512 tokens) \\
        & Max Sequence Length & 512 tokens \\
        & Training Objective & Mean Squared Error (MSE) on Log-Length \\
        & Optimization & AdamW (LR=2e-5, Weight Decay=0.01, Batch=32) \\
        & Training Schedule & 10 Epochs \\
        \midrule
        \multicolumn{3}{l}{\textit{\textbf{3. CLEAR Allocation Mechanism}}} \\
        \midrule
        & Global Parameters & $\alpha = 2.0$  \\
        & Optimization Method & Bisection Search (40 iterations, $\epsilon=1e-6$) \\
        \midrule
        \multicolumn{3}{l}{\textit{\textbf{4. Evaluation Scenarios}}} \\
        \midrule
        & Sample Size & $N=5,00$ queries per simulation stream \\
        & Evaluation Sources & MATH-500, OlympiadBench, AIME (24/25), AMC-23, Minerva \\
        \bottomrule
    \end{tabular}
    \end{small}
\end{table*}

\section{Appendix: Data Composition and Statistics}
\label{app:data_stats}

\paragraph{Training and Test Data} Table~\ref{tab:dataset_stats} and Table~\ref{tab:dataset_stats_30b} provide a detailed breakdown of the datasets used for training the threshold predictor and evaluating the budget allocation policies. To ensure the robustness of our \CLEAR{} framework, the evaluation suite is intentionally designed to be Out-of-Distribution (OOD) relative to the predictor's training set.

\begin{table*}[h]
    \centering
    \caption{\textbf{Detailed Statistics of Training and Evaluation Datasets for Qwen-2.5-math-7B-Instruct} under greedy decoding and \textbf{4K} new token constraints.}
    \label{tab:dataset_stats}
    \vspace{2mm}
    \renewcommand{\arraystretch}{1.2}
    \begin{small}
    \begin{tabular}{llcccccl}
        \toprule
        \textbf{Use Case} & \textbf{Dataset} & \textbf{Split} & \textbf{Total} & \textbf{Correct} & \textbf{Pass@1} & \textbf{Avg. Len} & \textbf{Threshold Tier} \\
        \midrule
        \multirow{2}{*}{\textit{Predictor Training}} & GSM8K & Train & 1,319 & 1,025 & 78.5\% & 77.71 & Easy \\
        & MATH & Train & 12,500 & 9,477 & 75.8\% & 658 & Moderate \\
        \midrule
        \multirow{6}{*}{\textit{Performance Eval.}} 
        & MATH 500 & Test & 500 & 354 & 70.8\% & 1,074 & Moderate \\
        & OlympiadBench & Test & 675 & 88 & 13.0\% & 1,059 & Hard \\
        & AMC-23 & Test & 40 & 19 & 47.5\% & 1,490 &  Hard \\
        & AIME 2024 & Test & 30 & 8 & 26.7\% & 1,369 & Very Hard \\
        & AIME 2025 & Test & 30 & 2 & 6.7\% & 1,886 & Very Hard \\
        & Minerva & Test & 272 & 35 & 12.9\% & 2,274 & Very Hard \\
        \bottomrule
    \end{tabular}
    \end{small}
\end{table*}

\begin{table*}[h]
    \centering
    \caption{\textbf{Detailed Statistics of Training and Evaluation Datasets for Qwen3-30B-A3B-Instruct} under greedy decoding and \textbf{16K} new token constraints. The threshold tiers are assigned based on the average reasoning length of these solutions.}
    \label{tab:dataset_stats_30b}
    \vspace{2mm}
    \renewcommand{\arraystretch}{1.2}
    \begin{small}
    \begin{tabular}{llcccccl}
        \toprule
        \textbf{Use Case} & \textbf{Dataset} & \textbf{Split} & \textbf{Total} & \textbf{Correct} & \textbf{Pass@1} & \textbf{Avg. Len} & \textbf{Threshold Tier} \\
        \midrule
        \multirow{2}{*}{\textit{Predictor Training}} & GSM8K & Train & 1,055 & 935 & 88.6\% & 1,254 & Easy \\
        & MATH & Train & 12,500 & 8,246 & 66.0\% & 3,763 & Moderate \\
        \midrule
        \multirow{6}{*}{\textit{Performance Eval.}} 
        & MATH 500 & Test & 500 & 288 & 57.6\% & 9,887 & Moderate \\
        & OlympiadBench & Test & 675 & 179 & 26.5\% & 9,450 & Hard \\
        & AMC-23 & Test & 40 & 14 & 35.0\% & 13,958 & Hard \\
        & AIME 2024 & Test & 30 & 3 & 10.0\% & 14,428 & Very Hard \\
        & AIME 2025 & Test & 30 & 14 & 46.7\% & 14,425 & Very Hard \\
        & Minerva & Test & 272 & 37 & 13.6\% & 10,771 & Very Hard \\
        \bottomrule
    \end{tabular}
    \end{small}
\end{table*}

\end{document}